\documentclass[twoside]{article}
\usepackage{arxiv}

\usepackage{amsthm}
\usepackage{amsmath}
\usepackage[square,numbers,sort&compress]{natbib} %
\usepackage{graphicx}

\usepackage{amssymb}
\usepackage{mathtools}
\usepackage{microtype}
\usepackage{booktabs}

\usepackage[utf8]{inputenc} 
\usepackage[T1]{fontenc}    
\usepackage{url}            
\usepackage{amsfonts}       
\usepackage{amssymb}        
\usepackage{pifont}         
\usepackage{nicefrac}       
\usepackage[dvipsnames]{xcolor}        
\usepackage{array}
\usepackage[english]{babel}

\usepackage{subfig}
\usepackage{bbm}
\usepackage{float}
\usepackage{comment}
\usepackage{soul}
\usepackage{algorithm, algorithmic}
\usepackage{multirow}
\usepackage{wrapfig}

\newcommand{\cmark}{\color{OliveGreen} \ding{51}}%
\newcommand{\xmark}{\color{BrickRed} \ding{55}}%
 
\theoremstyle{plain}

\newtheorem{corollary}{Corollary}
\newtheorem{proposition}{Proposition}

\newcommand{\bt}{\boldsymbol{\theta}}
\newcommand{\bx}{\mathbf{x}}

\newcommand{\bxobs}{\mathbf{x}_*}

\newcommand{\R}{\mathbb{R}}
\newcommand{\Z}{\mathbb{Z}}
\newcommand{\calUi}{\mathcal{U}}
\newcommand{\calPi}{\mathcal{P}}

\newcommand{\td}{\delta}

\newcommand{\bti}{\boldsymbol{\theta}^{(i)}}
\newcommand{\tdi}{\delta^{(i)}}

\newcommand{\N}{\mathcal{N}}

\newcommand{\bd}{\boldsymbol{\delta}}

\newcommand{\buq}{u_q}

\newcommand{\f}{f}

\newcommand{\ta}{\text{a}}

\newcommand{\bepsilon}{\boldsymbol{\epsilon}}

\newcommand{\bbeta}{\boldsymbol{\beta}}

\newcommand{\acq}{\mathcal{A}}

\newcommand{\bq}{\mathbf{q}}
\newcommand{\ba}{\mathbf{a}}
\newcommand{\bo}{\mathbf{o}}
\newcommand{\bb}{\mathbf{b}}

\begin{document}

\title{Likelihood-Free Inference in State-Space Models \\ with Unknown Dynamics}

\author{Alexander Aushev \\
  Helsinki Institute for Information Technology \\
  Department of Computer Science\\
  Aalto University, Finland \\
  \texttt{alexander.aushev@aalto.fi} \\
  \And
  Thong Tran \\
  Helsinki Institute for Information Technology \\
  Department of Computer Science\\
  Aalto University, Finland \\
  \texttt{thong.tran@aalto.fi} \\
  \And
  Henri Pesonen \\
  Department of Biostatistics \\
  University of Oslo, Norway \\
  \texttt{henri.pesonen@medisin.uio.no} \\
  \And
  Andrew Howes \\
  School of Computer Science \\
  University of Birmingham, UK \\
  \texttt{andrew.howes@aalto.fi} \\
  \And
  Samuel Kaski \\
  Helsinki Institute for Information Technology \\
  Department of Computer Science \\ 
  Aalto University, Finland \\ 
  Department of Computer Science \\ 
  University of Manchester, UK \\
  \texttt{samuel.kaski@aalto.fi} \\
}

\maketitle

\begin{abstract}
Likelihood-free inference (LFI) has been successfully applied to state-space models, where the likelihood of observations is not available but synthetic observations generated by a black-box simulator can be used for inference instead. However, much of the research up to now have been restricted to cases, in which a model of state transition dynamics can be formulated in advance and the simulation budget is unrestricted. These methods fail to address the problem of state inference when simulations are computationally expensive and the Markovian state transition dynamics are undefined. The approach proposed in this manuscript enables LFI of states with a limited number of simulations by estimating the transition dynamics, and using state predictions as proposals for simulations. In the experiments with non-stationary user models, the proposed method demonstrates significant improvement in accuracy for both state inference and prediction, where a multi-output Gaussian process is used for LFI of states, and a Bayesian Neural Network as a surrogate model of transition dynamics.
\end{abstract}

\section{Introduction}

Likelihood-free inference (LFI) methods \citep{sunnaaker2013approximate, sisson2018handbook, cranmer2020frontier} estimate parameters $\bt$ of a statistical model, given an observed measurement $\bxobs$ and a black-box simulator $g_{\bt}$. These methods use synthetic observations $\bx_{\bt} \sim g_{\bt} (\bx \mid \bt)$ produced by the simulator to assist the inference without requiring an analytical formulation of the likelihood $p(\bx \mid \bt)$. LFI has been successfully applied to identifying parameters of complex real-world systems, such as financial markets \citep{peters2012likelihood, barthelme2014expectation, ong2018likelihood}, species populations \citep{beaumont2002approximate, beaumont2010approximate, bertorelle2010abc} and cosmology models \citep{schafer2012likelihood, alsing2018massive, jeffrey2021likelihood}. A special type of applications of LFI are time-dependent systems, which can be described using state-space models (SSMs) \citep{kalman1960contributions, koller2009probabilistic} where observed measurements $\bx_t \in \R^n$ are emitted given a series of latent variables, the states $\bt_t \in \R^m$, as illustrated in Figure \ref{fig:ssm}.

Typically, state-space inference methods \citep{kalman1960contributions, anderson2012optimal, zerdali2017comparisons} require an observation model $g_{\bt}$ in the form of the likelihood $p(\bx_t \mid \bt_t)$ to find the posterior distribution $p(\bt_{1:T} \mid \bx_{1:T})$. When the observation model is unavailable, state-space learning methods \citep{frigola2014variational, melchior2019structured} are commonly used to infer $g_{\bt}$ from the observed time-series data. However, when $g_{\bt}$ is inferred, the states become very difficult to interpret for domain experts, since the states are no longer informed by a known model. An alternative solution to this problem is to use a simulator in place of $g_{\bt}$. LFI methods are able to infer the states and avoid learning $g_{\bt}$ by using a simulator as the observation model. Simulators are widespread in SSM settings \citep{ghassemi2017predicting, shafi2018estimating, georgiou2017adaptive} since they enable incorporation of additional prior knowledge about data-generating mechanisms without the need for a tractable likelihood $p(\bx_t \mid \bt_t)$. In this paper, we focus on LFI for SSMs, which fall under the category of approximate methods in the broader context of SSM inference.

An essential aspect of SSMs that is often overlooked in the LFI literature is the complexity of transition dynamics $h_{\bt_t}$. Current LFI methods for SSMs \citep{toni2009approximate, dean2014parameter} proceed by assuming dynamics to be either too simplistic (e.g. linear) or readily available for sampling, which cannot be applied to complex phenomena in meteorology \citep{errico2013development, zeng2020use}, cosmology \citep{lange2019cosmological, he2019learning} or behavioral sciences \citep{gimenez2007state, georgiou2017adaptive}. The underlying dynamics in these domains are often too complex and may greatly vary from one case to another (e.g. behaviour of different people). Unless the true data-generating process of transition dynamics are both linear and Gaussian, these LFI approaches would lead to poor state estimates and predictions. Even though there have been advances in LFI, for instance, in developing more efficient sampling-based methods \citep{jasra2012filtering}, proposing generation mechanisms of better-matching statistics \citep{martin2019auxiliary}, and establishing theoretical convergence guarantees \citep{dean2014parameter, martin2014approximate, calvet2015accurate}, they do not address this fundamental limitation.

\begin{figure}    
    \centering
    \includegraphics[width=0.45\textwidth]{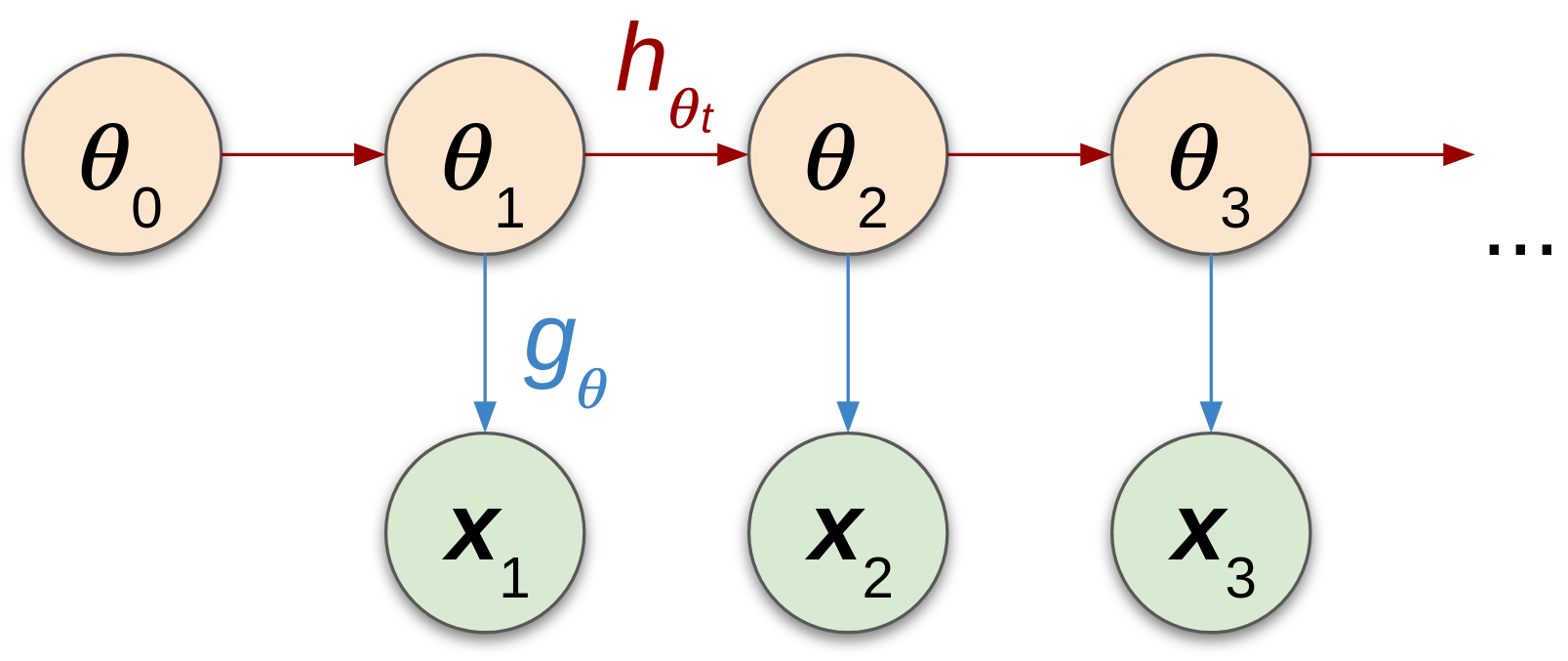}
    \caption{Graphical representation of an SSM. Latent states $\bt_t$ (orange) produce observations $\bx_t$ through the observation simulator $g_{\bt}$ (blue) and follow the Markovian transition dynamics $h_{\bt_t}$ (red).}
    \label{fig:ssm}
\end{figure}

In this paper, we introduce a method capable of likelihood-free state inference and state prediction in discrete-time SSMs. Our method operates in a LFI setting, where a time-series of observations $\bx_t$ and a simulator $g_{\bt}$ capable of replicating these observations are provided. The goal of the method is to infer the states $\bt_{1:T} = \{ \bt_1, ..., \bt_T \}$ which can produce the observed time-series $\bx_{1:T} = \{ \bx_1, ..., \bx_T \}$, using as few simulations as possible to reduce their potentially high computational cost. This setting is broader than is typically assumed by traditional LFI methods, since we do not assume the transition dynamics $h_{\bt_t}$ to be known (neither in its closed-form, nor its function family) or available for sampling, and also because the number of simulations can be limited to be small. Instead of assuming the transition dynamics, we learn a non-parametric model and use it as their \emph{surrogate} (or replacement) in state approximation and prediction. 

This paper contains three main contributions. First, we propose a solution to the previously unaddressed problem of state prediction in SSMs with unknown transition dynamics and limited simulation budget. We use samples from LFI approximations of state posteriors $p(\bt_t \mid \bx_t)$ to accurately model the state transition dynamics, with accuracy shown by empirical comparisons with state-of-the-art SSM inference techniques. Second, focusing on problems where LFI has to be \emph{sample-efficient}, i.e. the number of simulations needs to be reduced as much as possible, we improve upon the current LFI methods for the state inference task by leveraging time-series information. This is done by using a multi-objective surrogate for the consecutive states (e.g. for time-steps $j$ and $j+1$) and sampling from a transition dynamics model to determine where to next run simulations. Lastly, we demonstrate that the proposed method is needed to tackle the crucial case of user modelling, where user models are non-stationary because user's beliefs, preferences and abilities change over time.

\section{Background}
\label{sec:background}

Approximate Bayesian Computation (ABC) \citep{beaumont2002, csillery2010approximate, sunnaaker2013approximate} is arguably the most popular family of LFI methods. In its simplest variant, ABC with rejection sampling \citep{tavare1997inferring, pritchard1999population}, the simulator parameters are repeatedly sampled from the prior $p(\bt)$ to generate synthetic observations $\bx_{\bt}$. Then, $\bx_{\bt}$ are compared to the observed measurement $\bxobs$ with the so-called \emph{discrepancy} measure $\td(\bt) = \rho(\bxobs, \bx_{\bt})$, where $\rho(\cdot, \cdot)$ is a distance function, e.g. Euclidean. If synthetic observations $\bx_{\bt}$ have a smaller discrepancy than a user-defined threshold $\epsilon$, then they were produced by simulator parameters $\bt$ that could plausibly replicate the observed measurement $\bxobs$. This assumption, which is common for ABC approaches, results in the following approximations of the likelihood function $\mathcal{L}(\cdot)$ and the posterior $p(\bt | \bxobs)$:
\begin{align}
    \mathcal{L}(\bt) \approx \mathbb{E} [\kappa_\epsilon ( \td(\bt) )], \quad p( \bt | \bxobs) \propto \mathcal{L}(\bt) \cdot p(\bt).
    \label{eq:post}
\end{align}
Here $\kappa_\epsilon(\cdot)$ is the kernel with the maximum at zero, and whose bandwidth $\epsilon$ acts as an acceptance/rejection threshold. For instance, in ABC with rejection sampling, $\kappa_\epsilon(\td(\bt)) = \xi_{[0, \epsilon)](\td(\bt))}$, where $\xi_{[0, \epsilon)](\td(\bt))}$ equals one if $\td(\bt) \in [0, \epsilon)$ and zero otherwise. Unfortunately, ABC approaches need to simulate a lot of synthetic observations to get an accurate approximation of the posterior and therefore, unsuitable for inference with computationally heavy simulators.

Since many applications (including those considered in this paper) need to minimize the number of simulations, other methodologies, such as Bayesian optimization for LFI (BOLFI) \citep{gutmann2016bayesian} have emerged. In BOLFI, a Gaussian process (GP) surrogate is used for a discrepancy measure $\td(\bt)$, where the minimum of the GP surrogate mean function $\mu(\bt)$ can be used as $\epsilon$ and a Gaussian CDF $F( (\epsilon - \mu(\bt)) / \sqrt{\nu(\bt) + \sigma^2})$ with mean 0 and variance 1 as $\mathbb{E}[\kappa_\epsilon(\cdot)]$ in Equation \eqref{eq:post}. Here, $\nu(\bt) + \sigma^2$ is the posterior variance of the GP surrogate.

A main advantage of modelling the discrepancy with a GP is the uncertainty estimation. The GP's predictive mean $\mu(\bti)$ and variance $\nu(\bti)$ are used to calculate the utility (e.g., Expected Improvement, \citet{brochu2010tutorial}) of sampling the objective function at the next candidate point $\bt^{(i+1)}$, where $i$ denotes the number of a simulation. By maximizing this so-called acquisition function $\acq(\cdot)$ with respect to $\bt$, one chooses where to run simulations next. Because BOLFI actively chooses where to run simulations, its posterior approximation requires much fewer synthetic observations than other LFI methods that do not use active learning. However, BOLFI was not designed for SSMs, and hence does not make use of any temporal information as would be typical for SSMs to improve the quality of inference.

An alternative approach to sample-efficient LFI is global sequential neural estimation (SNE), which proceeds by learning the statistical relationship between observations and simulator parameters directly, through a neural network surrogate. This surrogate does not require retraining when the observation changes, if trained with a large enough sample set, which makes these SNE-methods especially suitable for a sequence of related inference tasks, such as required in time-series prediction. The SNE neural network can be used as a surrogate for the posterior, likelihood or likelihood ratio, resulting in SNPE \citep{papamakarios2016fast, goncalves2018flexible, greenberg2019automatic}, SNLE \citep{papamakarios2019sequential}, and SNRE \citep{durkan2020contrastive, hermans2020likelihood} methods respectively. These SNE methods address a more difficult problem than we do, of learning a model across all possible tasks (i.e. observed datasets). The price is that they require significantly more simulations than Bayesian optimization (BO) approaches, see Section 4.3 in \citep{aushev2020likelihood}.

Similarly to SNE approaches, \citet{fengler2021likelihood} introduce likelihood approximation networks (LANs), which approximate the likelihood for time-dependent generative models and apply them in dynamical systems in cognitive neuroscience. The key distinction of their approach is the assumption that the time-component is one of the inputs of the observation model, which allows them to learn the observation model at an arbitrary time-step. This assumption shifts the role of the dynamics onto the observation model, which is often beneficial for diffusion models \citep{reynolds2009levy, wieschen2020jumping}, but not for the models of human behaviour \citep{schall2019accumulators, futrell2020lossy, pothos2002simplicity}. In contrast, our approach does not rely on the explicit dependency of the observation model on time, which enables state predictions in cases when the transition dynamics are unknown at the cost of amortization. 

The issue of handling non-linear transition dynamics, in general, has been primarily addressed outside of the LFI literature. This large and growing set of methods includes extended Kalman filters \citep{anderson2012optimal, zerdali2017comparisons}, GP-SSMs \citep{frigola2014variational, melchior2019structured}, sequential Monte Carlo \citep{doucet2001sequential, smith2013sequential, septier2013bayesian} and Bayes filtering \citep{smidl2008variational, karl2016deep}. Although they are not directly applicable to the LFI setting considered in this paper, we summarised them in Table \ref{tab:relatedworks} along with the relevant LFI literature to highlight important connections.

\begin{table}
    \centering
    \footnotesize
    \begin{tabular}{m{0.2\textwidth}ccccc}
        \hline
        \textbf{Reference} & \textbf{Method} & \textbf{LFI} & \textbf{Amortized} & \textbf{\# Simulations} & \textbf{Dynamics} \\ \hline
        
        \citet{rubin1984bayesianly} & ABC & \cmark & \xmark & {\color{BrickRed} 	$\approx$10k} & \xmark \\
        \citet{gutmann2016bayesian} & BOLFI & \cmark & \xmark & {\color{OliveGreen} 	$\approx$50} & \xmark \\
        \citet{papamakarios2016fast} & SNEs & \cmark & \cmark & {\color{BurntOrange}	$\approx$1k} & \xmark  \\
        \citet{fengler2021likelihood} & LANs & \cmark & \cmark & {\color{BurntOrange}	$\approx$1k} & \xmark  \\
        \citet{jasra2012filtering} & ABC/Filt. & \cmark & \xmark & {\color{BrickRed} 	$\approx$10k} & {\color{BurntOrange} linear}  \\
        \citet{izenman1988introduction} & Kalman & \xmark & \xmark & - &  {\color{BurntOrange} linear}  \\
        \citet{anderson2012optimal} & Ext. Kalman & \xmark & \xmark & - & {\color{OliveGreen} non-linear}  \\
        \citet{doerr2018probabilistic} & PR-SSM & \xmark & \xmark & - & {\color{OliveGreen} non-linear}\\
        \citet{ialongo2019overcoming} & GP-SSM & \xmark & \xmark & - & {\color{OliveGreen} non-linear}  \\
        
        This work & LMC-BNN & \cmark & \xmark & {\color{OliveGreen}$\approx$50} & {\color{OliveGreen} non-linear}  \\
    \end{tabular}
    \caption{Comparison of inference methods in SSMs with references to selected representative works. LFI methods use simulators to infer states (or simulator parameters), and can be either amortized (do not need to be retrained when the observations change) or non-amortized (valid only for the observed data). These methods also vary in the number of simulations required for inference and the type of dynamics these methods assume. We report the single observation budget (per a time-step) for non-amortized methods and the total budget for amortized methods.}
    \label{tab:relatedworks} 
\end{table}

\section{Likelihood-free inference in state-space models}

In this section, we introduce a multi-objective approach to LFI in SSMs, which improves sample-efficiency of existing methods by using the model for discrepancy shared across consecutive states while also learning the model of the transition dynamics. The main elements of the solution are presented in Figure \ref{fig:overview}. To estimate state points $\bt_t$, given $\bx_t$, we employ a multi-objective surrogate $\widetilde{\delta}_{\bt}$ for discrepancies, and then approximate the posterior over states $p(\bt_t \mid \bx_t)$ with \eqref{eq:post}. At the same time, we randomly pair consecutive posterior samples $(\bt_j, \bt_{j+1})$ and train a non-parametric surrogate for the state transition $\widetilde{h}_{\bt_t}$, whose predictive posterior $p(\bt_{t+1} \mid \bx_t)$ proposes candidates for future simulations. We summarize our approach in Algorithm \ref{alg:lmcbnn}, where $\bt_*$ denotes simulator parameter points shared across all time-steps.

\begin{figure}
    \centering
    \includegraphics[width=0.5\textwidth]{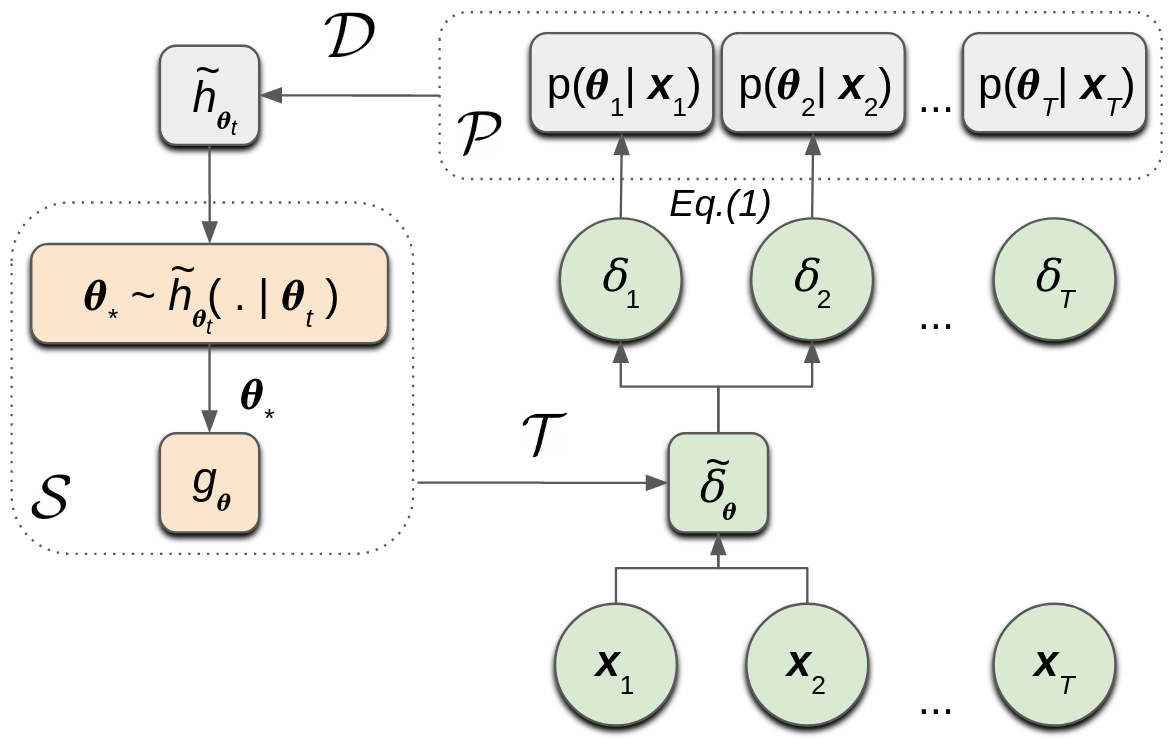}
    \caption{An overview of our approach, in which the $\widetilde{\delta}_{\bt}$ surrogate is used for LFI of states and $\widetilde{h}_{\bt_t}$ for the unknown transition dynamics. The $\widetilde{\delta}_{\bt}$ models the corresponding discrepancies $\delta_t \equiv \delta_t(\bt_*)$ of several observations (green) inside a moving window (here, with the size of two), from which posteriors are extracted according to Equation (1) in $\mathcal{P}$. $\widetilde{h}_{\bt_t}$ is trained with paired samples $\mathcal{D}$ from posteriors of consecutive states (grey); its predictive samples are used as proposals (orange) for  simulations $\mathcal{S}$.}
    \label{fig:overview}
\end{figure}

\begin{algorithm}
   \caption{Multi-Objective LFI with Transition Model}
   \label{alg:lmcbnn}

\begin{algorithmic}
    \STATE {\bfseries Input:} observations $\{\bx_1, ..., \bx_T \}$, observation simulator $g_{\bt}$, moving window size $L$, number of initial simulations $B_0$, simulation budget $B_\text{sim}$, number of posterior samples $I$ for states and $K$ for states predictions, prior over states $p(\bt_*)$;
    \STATE {\bfseries Output:} state posteriors $\mathcal{P}$, transition model $\widetilde{h}_{\bt_t}$;
    \STATE
    \STATE Initialize an empty set for state posteriors $\mathcal{P} := \emptyset$; 
    \STATE Initialize the transition model $\widetilde{h}_{\bt_t}$ (see Section \ref{sec:learningstate});
    \STATE Simulate $B_0$ observations with parameters sampled from the prior $p(\bt_*)$;
    \STATE Form the set $\mathcal{S}$ with the resulting data $\{(\bt_*, \bx_{\bt})\}_{b=1}^{B_0}$;
    \STATE Initialize the start and end indexes for the moving window: $t_0 := 0$, $t := L$;
    \WHILE{$t < T$}
        \STATE Calculate discrepancies for $\mathcal{S}$ inside the moving window $w:=t_0$:$t$;
        \STATE Form the training set $\mathcal{T}$ with the resulting data $\{( \bt_*, \delta_w (\bt_*))\}_{s=1}^{|\mathcal{S}|}$;
        \STATE Train the multi-objective model $\widetilde{\delta}_{\bt}$ from Section \ref{sec:mostateinference} with $\mathcal{T}$;
        \STATE Extract $I$ posterior samples over states from $\widetilde{\delta}_{\bt}$ with \eqref{eq:post};
        \STATE Replace old posterior estimates in $\mathcal{P}$ with the resulting samples $\{\bt_w\}_{i=1}^I$;
        \STATE Randomly sample $K$ consecutive (some $j$ and $j+1$) state samples from $\mathcal{P}$;
        \STATE Form the training set $\mathcal{D}$ with the resulting data $\{\bt_j, \bt_{j+1}\}_{k=1}^K$; 
        \STATE Update parameters of $\widetilde{h}_{\bt_t}$ with data from $\mathcal{D}$ through training;
        \STATE Simulate $B_\text{sim}$ parameters proposed by $\widetilde{h}_{\bt_t}$ with \eqref{eq:pred};
        \STATE Augment $\mathcal{S}$ with the resulting new data $\{(\bt_*, \bx_{\bt})\}_{b=1}^{B_\text{sim}}$;     
        \STATE Move the moving window by adjusting its indexes: $t_0 := t_0 + 1$, $t := t + 1$;
    \ENDWHILE
\end{algorithmic}
\end{algorithm}

\subsection{Multi-objective state inference}
\label{sec:mostateinference}

As an extension to BOLFI, we employ a multi-objective surrogate model for the discrepancies $\delta_{t}(\bt_*) = \rho(\bx_t,  \bx_{\bt})$ at different $t$, thus considering multiple discrepancy objectives simultaneously and leveraging information between consecutive states. More specifically, we pass discrepancies of the consecutive states to the surrogate separately (e.g. $\delta_{t-1}(\cdot), \delta_{t}(\cdot))$, but through the use of shared parameters of the multi-objective surrogate they become associated. This approach allows using a discrepancy model of the previous state to infer the current state, instead of simply discarding it. Moreover, it allows having a much more flexible surrogate for LFI of states than the traditional GP used in BOLFI. These changes do not need any additional data to fit the surrogates, because all synthetic observations $\bx_{\bt}$ for discrepancy objectives can be shared across all states (therefore, we use $\bt_*$ instead of $\bt_t$ in the context of simulations). When we consider a new observation $\bx_{t+1}$, we simply need to recalculate the discrepancy values for all synthetic observations. Once we have a trained surrogate for discrepancy objectives, we infer state posteriors $p(\bt_{t} | \bx_{t})$, similarly as in BOLFI. 

There is an additional challenge for adapting multi-objective surrogates in SSMs, namely high computational cost associated with considering too many objectives. The time-series can potentially have hundreds of time-points, and expanding the number of considered objectives may be detrimental for the performance of the surrogate. We avoid this problem by limiting the number of objectives the surrogate can have. Instead of considering all available time-steps as objectives, we propose to consider only $L$ recent objectives by gradually including new objectives and discarding old ones that have little impact on current states. The size of this \emph{moving window} depends on how rapidly the transition dynamics change. As the size of the window $L$ grows, the model becomes less sensitive to the noise from the dynamics, at the cost of increased computations and decreased adaptability to the most recent state transitions. Overall, the moving window reduces the number of objectives $L$ considered at a time, making multi-objective modelling in the SSM setting feasible. In Appendix A, we further investigate the influence of the moving window size hyperparameter on state inference and prediction, and show that having only two objectives ($L=2$) is the most beneficial choice in terms of the quality of posterior approximations and low computational time.

\subsection{Learning state transition dynamics}
\label{sec:learningstate}

While we gradually improve LFI posterior approximations $p(\bt_t | \bx_t)$ by acquiring new simulations, we use empirical samples from the latest available approximations to learn a stochastic model of transition dynamics. This model should be able to learn from noisy samples of LFI posterior approximations $p(\bt_t | \bx_t)$, and be flexible enough to fit arbitrary function families the dynamics may follow. In addition, it should be able to handle uncertainty associated with samples outside the training distribution, as samples from posterior approximations tend to be concentrated around the main mode of the learning data. For these reasons, the appropriate transition model should be Bayesian and non-parametric (or semi-parametric). Such a model would take the uncertainty associated with posterior approximations into account and be flexible enough to follow possibly non-linear transition dynamics.

We propose to train this model in an autoregressive fashion, by forming a training set of $K$ randomly paired sample points from posteriors (e.g. $p(\bt_{t-1} \mid \bx_{t-1})$, $p(\bt_t \mid \bx_t)$). More specifically, we assume the Markov property in the transition dynamics and use pairs of states instead of their whole trajectories. For each SSM time interval, we group consecutive state posterior samples in a training set, and expand it when new state posteriors become available (as we move forward in time). Thus, the transition model does not need to be retrained when new observations present themselves and can be actively used throughout state inference for determining where to run simulations next. This can be done by sampling the predictive posterior $p(\bt_{T+1} | \bx_{T})$ from the trained model $\widetilde{h}_{\bt_T}$: 

\begin{align}
    p(\bt_{T+1} | \bx_{T}) = \int \widetilde{h}_{\bt_T}(\bt_{T+1} | \bt_{T}) p(\bt_{T} | \bx_{T}) d\bt_{T}. \label{eq:pred}
\end{align}

This way, the state transition model $\widetilde{h}_{\bt_t}$ does not inform state posteriors directly, but only provides simulation candidates for the LFI surrogate. Ultimately, accumulating more simulations improves the discrepancy surrogate for the LFI of states and, by extension, quality of posterior samples, while higher-quality posterior samples allow more accurate learning of state transition dynamics.

\subsection{Computational complexity and model choices}
\label{sec:complexity}

The proposed multi-objective approach to LFI requires some choice of the surrogate models. Here, we present the model choices for our approach, followed by a resulting complexity analysis of the Algorithm \ref{alg:lmcbnn}.

Following the requirements for the surrogates from Sections \ref{sec:mostateinference} and \ref{sec:learningstate}, we chose a linear model of coregionalization (LMC) \citep{fanshawe2012bivariate} as a multi-objective surrogate for discrepancies and a Bayesian Neural Network (BNN) \citep{kononenko1989bayesian, esposito2020blitzbdl} as a surrogate for state transition dynamics. LMC is one of the simplest multi-objective models; it expresses each of its $L$ outputs $\f_l$ as a linear combination $f_l(\bt_*) = \sum_{q=1}^Q \ta_{l,q} \buq$, as shown in Figure \ref{fig:lmc}, where the $\buq \sim GP(0, \nu(\bt_*))$ are latent GPs and the $\ta_{l,q}$ are linear coefficients that need to be solved. As for BNN, it can be represented as an ensemble of neural networks, where each has its weights $\omega^{(h)}$ drawn from a shared, learnt probability distribution \citep{blundell2015weight} with $\omega^{(h)} \sim \mathcal{N}(\mu^{(h)}, \log(1 + \chi^{(h)}))$, where $\mu^{(h)}$ and $\chi^{(h)}$ are the hyperparameters that need to be learned. Previously, neural networks have been successfully applied in SSM settings for either modelling the transition dynamics or the observation model \citep{rivals1996black, bonatti2021one}.

\begin{figure}
    \centering
    \includegraphics[width=.35\textwidth]{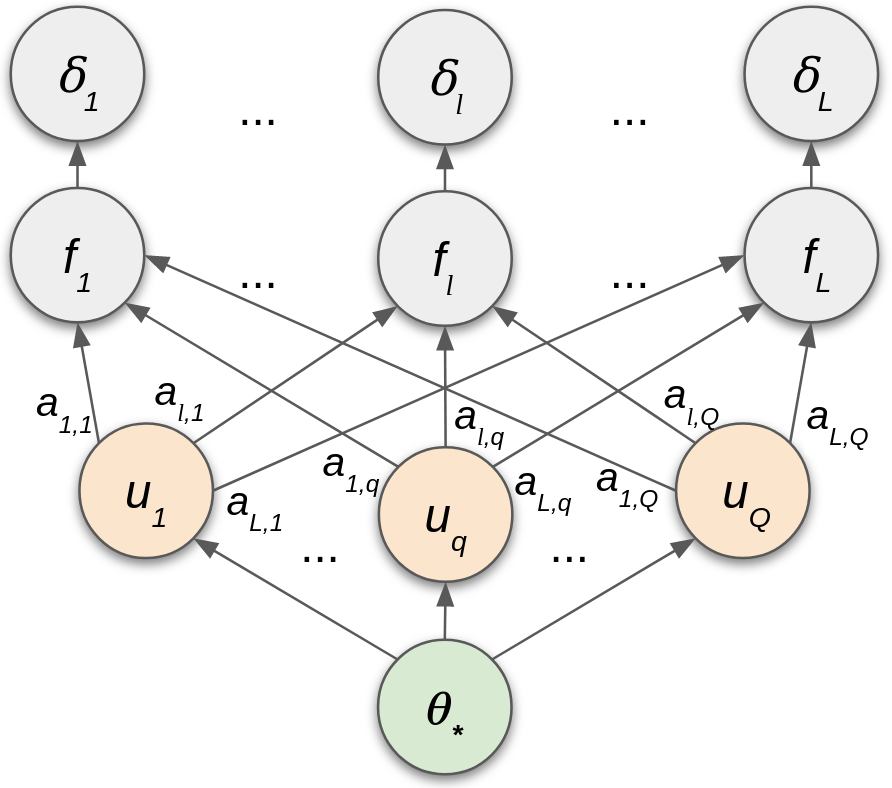}
    \caption{Graphical representation of the LMC. The discrepancy outputs $\delta_t \equiv \delta_t(\bt_*)$ are modelled as a linear combination of latent functions $\buq$. The model shares the same parameter values $\bt_*$ between all objectives.}
    \label{fig:lmc}
\end{figure}

Given the aforementioned model choices, the resulting computational complexity of the Algorithm \ref{alg:lmcbnn} is dominated by three major stages: training of the multi-objective surrogate $\widetilde{\delta}_{\bt}$, posterior extraction from discrepancy surrogates (Equation \eqref{eq:post}) and training of the transition dynamics model $\widetilde{h}_{\bt_t}$. Both LMC and BNN are trained by minimizing the variational evidence lower bound (see more details about in Appendix C). Starting with the cost of training $\widetilde{\delta}_{\bt}$, it depends on the number of synthetic observations $|\mathcal{S}|$ (the cardinality of $\mathcal{S}$), on the size of the moving window $L$ and on the user-specified number $M$ of inducing points \citep{alvarez2011computationally} for the LMC, resulting in the complexity $\mathcal{O}(|\mathcal{S}| L M^2)$, as opposed to $\mathcal{O}(|\mathcal{S}| M^2)$ by traditional GPs used in BOLFI. Moving on to the posterior extraction, this stage consists of finding the appropriate $\epsilon$ (e.g. by minimizing the GP mean function), and then applying Equation \eqref{eq:post}, which is bounded by the complexity of calculating the variance of the surrogate for each of the $I$ samples from the posterior over states, resulting in $\mathcal{O}(L M^2 I)$. Finally, if we apply variational inference \citep{zhang2018advances} to train the transition dynamics model  $\widetilde{h}_{\bt_t}$, the computational cost is linear in the number $W$ of BNN parameters  --  $\mathcal{O}(W K E S_p)$, where $K$ is the overall amount of training data for $\widetilde{h}_{\bt_t}$, $E$ is the number of epochs, and $S_p$ is the number of parameter samples from the posterior distribution that is required to get the distribution of outputs. Depending on the choice of hyperparameters, the computational complexity of the Algorithm \ref{alg:lmcbnn} is bounded by either $\mathcal{O}(|\mathcal{S}| L M^2)$, $\mathcal{O}(L M^2 I)$ or $\mathcal{O}(W K E S_p)$. Most of these parameters are common in LFI (e.g. $|\mathcal{S}|, I$), and the rest are specific to surrogate choices, which can be replaced with fewer-parameter alternatives if needed. We provide recommendations for choosing these hyperparameters in Appendix C.

\subsection{Theoretical properties}
\label{sec:theory}
 
\paragraph{Convergence} Here, we show that our method learns a reasonable approximation of states and transition dynamics when provided with sufficient data. The state approximations for $p( \bt_t | \bx_t )$ are obtained through the unnormalized likelihood function in \eqref{eq:post}, which was shown in Proposition \ref{th:prop-1} of \citet{gutmann2016bayesian} to be a non-parametric approximation of the true likelihood:

\begin{proposition}
\label{th:prop-1}
Maximizing the synthetic likelihood $\mathcal{L}(\bt_*)$ in \eqref{eq:post} corresponds to maximizing a lower bound of a non-parametric approximation of the log likelihood, when the kernel function $\kappa_\epsilon (\cdot)$ is convex.

\[\text{log }\mathcal{L}(\bt_*) \geq \mathbb{E} [\kappa_\epsilon ( \td_t(\bt_*) )]\]
\end{proposition}

It is straightforward to show that for our LMC model Proposition \ref{th:prop-1} also holds when its kernel is a Gaussian CDF and the model simulates the empirical discrepancy.

\begin{corollary}
\label{th:corol-1}
Proposition \ref{th:prop-1} holds for the LMC model of discrepancy, assuming the Gaussian CDF kernel $F( (\epsilon - \mu(\bt_*)) / \sqrt{\nu(\bt_*) + \sigma^2})$ from Section \ref{sec:background}.
\end{corollary}

\begin{proof}
According to properties of the Gaussian CDF, the kernel $F(\cdot)$ is convex, which does not change the inequality expression in Proposition \ref{th:prop-1}. This inequality holds since the Jensen’s inequality yields a lower bound for both $\mathcal{L}(\bt_*)$ and its logarithm when the functions are convex. The argument of $\kappa_\epsilon(\cdot)$ also remains the same, since the LMC models the empirical discrepancy.
\end{proof}

As for the approximations of state transitions $p(\bt_{t+1} | \bt_{t})$, its convergence follows from the universal  approximation theorem of neural networks \citet{hornik1989multilayer}, according to which every non-linear function can be represented by a neural network containing a single hidden layer composed of neurons whose transfer function is bounded. This also applies in our case, since we use the BNN model for transition dynamics. Following the central limit theory, the expectation of our approximation converges to the target distribution provided sufficient number of parameters and data.

\paragraph{Restrictions on the class of systems for modelling} Our model choices impose additional limitations on the class of systems that are difficult to model with our framework. The first limitation is high predictive variance when learning systems with long-term dependencies, as we train the BNN using a single trajectory of few observations (50 in our experiments from Section \ref{sc:experiments}). The BNNs are very flexible, however, such lack of data can make them a poor option when it comes to modelling systems with non-linear dynamics and long-term dependencies. Aside from that, BNNs do not pose additional theoretical restrictions on the class of systems for modelling. The second limitation is on the type of the observational distribution, which is modeled in our framework by the LMC. While LMCs are much more flexible than vanilla GPs, they may also struggle with modelling asymmetric, skewed or multimodal noise in the observation model. The reason for that is that these GP-based surrogates assume all noise to be Gaussian, making the state posterior approximations less reliable in cases where this assumption is violated. The third limitation also stems from using GPs in LFI, which suffers from the curse of dimensionality limiting the dimension of observation model to be lower than 10. On other hand, GPs allow our framework to be very sample-efficient, requiring only few synthetic observations to approximate the likelihood. Nevertheless, when the simulation budget for the observation model is not restricted to the order of hundred simulations, we recommend using more complex surrogates, such as SNEs or LANs from Table \ref{tab:relatedworks}, along with our approach to modelling state transitions.

\section{Experiments}
\label{sc:experiments}

We assess the quality of our method for state inference and prediction tasks in a series of SSM experiments, where a simulator serves as the observation model $g_{\bt}$. In the experiments, our method uses the surrogate choices of LMC and BNN, as was described in Section \ref{sec:complexity}. We demonstrate that it can accurately learn state transition dynamics and improve upon existing LFI methods for the state inference task. Moreover, we investigate sample-efficiency of the proposed method and demonstrate its effectiveness in non-stationary user modelling case studies. We compare our method against traditional SSM methods in cases with available closed-form likelihoods and against LFI methods when only a simulator is available and traditional methods cannot be applied.

\subsection{Experimental setup}
\label{sc:exp-setup}

We simulated time series of observations based on single sampled trajectories from ground-truth transition dynamics (available for the evaluation purposes, but unknown to the methods) of five SSMs, described in Section \ref{sec:ssms}. Our goal was to estimate the simulator parameters that likely produced these observations, and learn the model of transition dynamics for state prediction based on the sampled trajectory.

For the state inference task, we compare the quality of state estimates by our approach against other LFI methods: BOLFI \citep{gutmann2016bayesian}, SNPE \citep{papamakarios2016fast}, SNLE \citep{papamakarios2019sequential}, and SNRE \citep{durkan2020contrastive}. We use a fixed simulation budget for all these methods, with 20 simulations to initialize the models, and then with 2 additional simulations for each new time-step. For the SNE approaches (SNPE , SNLE, SNRE), we provided all simulations at once, since that is their intended mode of operation. As for the prediction task, we sampled state trajectories from the transition model and evaluated them against trajectories from the ground-truth dynamics. We performed these experiments in SSMs with simulators that have tractable likelihoods $p(\bx_t \mid \bt_t)$, providing the closed-form of the ground-truth likelihoods to the state-of-the-art SSM inference methods GP-SSM \citep{frigola2014variational, ialongo2019overcoming} and PR-SSM \citep{doerr2018probabilistic}, while our method was still doing LFI. For all methods in the prediction task, we provided 50 observations, and then sampled trajectories that had the same length of 50 time-steps. 

We also compared two variants of our method that differ only in the way the next simulations are sampled: LMC-BLR, where samples were taken from Bayesian linear regression (BLR) models that linearized the transition dynamics along 50 observed time-steps; and LMC-qEHVI, where a popular acquisition function for multitask BO, q-expected hypervolume improvement (qEHVI) \citep{daulton2020differentiable}, was used to provide samples. The role of these variants was to evaluate how the choice of the future simulations impacts the quality of state inference and prediction.

All models were assessed in terms of the root-mean-squared error (RMSE) between the state estimates and their ground-truths. The experiments were repeated 30 times with different random seeds. Additional details on implementation of the methods can be found in Appendix C; all code for replicating the experiments is included in the Supplement.

\subsection{The state-space models}
\label{sec:ssms}

In this section, we present two case studies with non-stationary user models and three SSMs with tractable likelihoods. In user modelling experiments, we simulated behavioural data from humans who completed a certain task in two different experiments, described in Sections \ref{sc:umap} and \ref{sc:gaze}. For the first task, the user evaluated dataset embeddings for a classification problem, and the evaluation score was used as behavioural data. During the second task, the user searched for a target on a display and the search time was measured. Our task in the experiments was to track the changing parameters of user models and learn their dynamics.

In addition to non-stationary user models, we also experimented with three models with tractable likelihoods, common in SSM literature: linear Gaussian (LG), non-linear non-Gaussian (NN) and stochastic volatility (SV) models. In the LG model, the state transition dynamics and observation model are both linear, with high observational white noise. The NN model is a popular non-linear SSM \cite{kitagawa1996monte}, where each observation has two unique solutions. Lastly, we used the SV model \cite{barndorff2002econometric}, which is used for predicting volatility of asset prices in stock markets \citep{taylor1994modeling,shephard1996statistical}. We provide full descriptions of these SSMs in Appendix B.

\subsubsection{UMAP parameterization}
\label{sc:umap}

With the first non-stationary user model, we infer uniform manifold approximation and projection (UMAP) \citep{mcinnes2018umap} parameters that best satisfy the human subject's needs. We assume that the subject uses UMAP to generate low-dimensional data representations for a classification task, and that their preferences change with time. Specifically, in the beginning, the subject does not have any prior knowledge of the dataset, so data exploration takes the priority. As they become more familiar with the data, the priority gradually shifts from exploration to maximizing classification accuracy of their model. Our hypothesis is that by modelling the change in subject's preferences, we can anticipate the preferences and propose good UMAP parameter candidates faster to the human user.

For this experiment, the subject's needs are simulated by an evaluation function that takes the embedding of a handwritten digit dataset \citep{alpaydin1998cascading} as an input and assigns a corresponding preference score. The evaluation function is based on the weighted sum of the density based cluster validity (DBCV) score $\calUi(\cdot)$ \citep{moulavi2014density} and the c-support vector classification (SVC) \citep{boser1992training, cortes1995support} accuracy $\calPi(\cdot)$. Both functions depend on the UMAP parameter values $\bt_* = \{\theta_{d}, \theta_{dist}, \theta_{n}\}$ which need to be inferred: the dimension of the target reduced space $\theta_{d}$, an algorithmic parameter $\theta_{dist}$ that controls how densely the points packed, and the neighbourhood size $\theta_{n}$ to use for local metric approximation. By regulating the weight $w_t$, one can prioritize data exploration $\calUi$ or maximization of a classification accuracy $\calPi$ with 

\begin{align}
    \td_{t}(\bt_*) &= (1 - w_t) \cdot \calUi(\bt_*) +  w_t \cdot \calPi(\bt_*) \label{eq:umap-eval}, \quad w_t = \frac{1}{1 + e^{-0.1(t - 25)} }.
\end{align}

We make a few simplifying modelling choices which are justified by cognitive science literature \citep{slovic2002rational, lichtenstein2006construction}, as discussed below. We assume the subject cannot explicitly specify the ideal embedding, and hence the SSM observations $\bx_t$ are latent. But since the subject has the ability to evaluate embeddings with the preference score, we model the preference score $\td_{t}$ directly as an objective. In summary, the UMAP algorithm serves as an observation model and the transition dynamics are unknown, but implicitly regulated through \eqref{eq:umap-eval}. We use uniform priors for states throughout the experiments:

\begin{align}
    \theta_{d} &\sim \text{Unif}(1, 64) \in \Z, \quad \theta_{dist} \sim \text{Unif}(0, 0.99) \in \R, \nonumber  \\ 
    \theta_{n} &\sim \text{Unif}(2, 200) \in \Z. \nonumber
\end{align}

Note, that the ground truth values for this SSM are unknown, and we approximate them with ABC with rejection sampling allowing it use significantly more simulations than our methods uses. This procedure along with details on the implementation of the SSM are described in Appendix C.

\subsubsection{Eye movement control for gaze-based selection}
\label{sc:gaze}

For the second non-stationary user model, we infer properties of human gaze in a series of simulated eye movement control trials \citep{ zhang2010modeling, schuetz2019explanation}. In these trials, the user model of the human searches for a target on a 2D screen by performing eye movements (saccades), based on its beliefs about the target location and information from peripheral vision. When the gaze location matches the location on the screen, the task is considered complete and a new target appears. As the human subject performs more tasks, they get weary, which results in an increasing latency between saccades $\theta_{l}$. We believe that modelling the dynamics of eye movement latency will allow robust inference of the individual characteristics that are independent of weariness.
 
The subject needs several moves to reach the target because the actions and observations are corrupted with two noise sources: the ocular motor noise $\theta_{om}$ and the spatial noise of peripheral vision $\theta_{s}$. The quantitative values for these two variables vary for each individual, while the saccade latency $\theta_{l}$ varies between different trial instances. We assess the performance of the subject based on total time $x \in \R$ it took for the gaze to reach the target,

\begin{align}
    x_t &= \sum_{e} (2.7 \cdot \hat{A}^{(e)}(\theta_{om}, \theta_{s}) + \theta_{l, t}), \quad \theta_{l,t} = 12 \cdot log(t + 1) + 37. \nonumber
\end{align}

Here, the $x_t$ are SSM observations, $\bt_* = \{ \theta_{om}, \theta_{s}, \theta_{l}\}$ are state values that need to be inferred, $\hat{A}^{(e)}$ is the saccade amplitude function, and $e$ is the saccade index. The increase of the latency $\theta_{l}$ serves as state transition dynamics $h_{\bt}$ that needs to be modelled for prediction, while eye movement control task environment is considered as an observation model $g_{\bt}$. To produce behavioural data, ground truth values of 0.01 and 0.09 for $\theta_{om}$ and $\theta_{s}$ respectively were used in the observation model. Finally, we calculate a Euclidean distance directly applied to $x$ as a discrepancy and use the following priors for the states:

\begin{align}
    \theta_{om} &\sim \text{Unif}(0, 0.2) \in \R, \quad \theta_{s} \sim \text{Unif}(0, 0.2) \in \R, \nonumber \\
    \theta_{l} &\sim \text{Unif}(30, 60) \in \R. \nonumber
\end{align}

The user model was implemented with reinforcement learning by \citep{chen2021adaptive}. More details on implementation can be found in Appendix C.

\subsection{Results and analysis}

\begin{table*}
    \centering
    \caption{Comparison of LFI methods (rows) in different SSMs (columns) for the state inference task. The performance was measured with $95\%$ confidence interval (CI) of the RMSE between estimated vs ground truth state values for 50 time-steps. The best results in each column are highlighted in bold.}
    \begin{tabular}{cccccc}
        \\
        \hline
        \textbf{Methods} & \textbf{LG} & \textbf{NN} & \textbf{SV} & \textbf{UMAP} & \textbf{Gaze} \\  \hline
        LMC-BNN & 1.77 $\pm$ 0.12 & 6.92 $\pm$ 0.51 & 16.14 $\pm$ 3.27 & \textbf{58.24 $\pm$ 3.62} &  58.7 $\pm$ 5.4 \\
        LMC-BLR & \textbf{1.3 $\pm$ 0.1} & \textbf{6.86 $\pm$ 0.54} & \textbf{13.15 $\pm$ 3.25} & 59.19 $\pm$ 3.31 & 60.6 $\pm$ 5.8 \\
        LMC-qEHVI & 1.5 $\pm$ 0.1 & 7.03 $\pm$ 0.55 & 24.4 $\pm$ 2.5 & 64.96 $\pm$ 3.72 & \textbf{56.9 $\pm$ 4.5 } \\
        BOLFI & 1.55 $\pm$ 0.15 & 7 $\pm$ 0.6 & 27.35 $\pm$ 3.45 & 84.31 $\pm$ 3.54 & 73.45 $\pm$ 5.75 \\
        SNPE & 7.15 $\pm$ 0.65 & 18.2 $\pm$ 0.93 & 77.4 $\pm$ 3.1 & 74.13 $\pm$ 3.21 & 68.1 $\pm$ 7.8 \\
        SNLE & 6 $\pm$ 0.5 & 10.35 $\pm$ 0.64 & 54.25 $\pm$ 2.45 & 71.45 $\pm$ 3.44 & 77.25 $\pm$ 4.05 \\
        SNRE & 10.4 $\pm$ 1.7 & 17.93 $\pm$ 1.34 & 57.15 $\pm$ 2.35 & 75.85 $\pm$ 1.26 & 80.75 $\pm$ 1.35 \\
    \end{tabular}
    
    \label{tab:Table-1}
\end{table*}

\begin{table*}
    \centering
    \caption{Comparison of transition dynamics models (rows) in different SSMs (columns). The performance was measured with $95\%$ CI of the RMSE between sampled vs ground truth trajectories of length 50. The best results in each column are highlighted in bold.}
    \begin{tabular}{cccccc}
    \\
        \hline
        \textbf{Methods} & \textbf{LG} & \textbf{NN} & \textbf{SV} & \textbf{UMAP} & \textbf{Gaze} \\ \hline
        LMC-BNN & 210 $\pm$ 4 & \textbf{ 148 $\pm$ 2} & 117 $\pm$ 21 & \textbf{1394 $\pm$ 27} & \textbf{1365 $\pm$ 3} \\ 
        LMC-BLR & \textbf{ 64 $\pm$ 7} & 154 $\pm$ 4 & \textbf{100 $\pm$ 37} & 1409 $\pm$ 49 & 1372 $\pm$ 3 \\
        GP-SSM & 284 $\pm$ 71 & 2204 $\pm$ 1111 & 3206 $\pm$ 1175 & - & - \\
        PR-SSM & 253 $\pm$ 68 & 610 $\pm$ 510 & 1378 $\pm$ 740 & - & - \\
    \end{tabular}
    
    \label{tab:Table-2}
\end{table*}

\begin{table*}
    \centering
    \caption{Time comparison of LFI methods (rows) in different SSMs (columns) for training 50 time-steps. The running time is shown in minutes along with $95\%$ CI. The best results in each column are highlighted in bold.}
    \begin{tabular}{cccccc}
        \\
        \hline
        \textbf{Methods} & \textbf{LG} & \textbf{NN} & \textbf{SV} & \textbf{UMAP} & \textbf{Gaze} \\  \hline
        LMC-BNN & 87.6 $\pm$ 2.6 & 79.2 $\pm$ 1 & 81 $\pm$ 2.9 & 171.2 $\pm$ 5 & 408 $\pm$ 8 \\
        LMC-BLR & 82.1 $\pm$ 5.5 & 48.2 $\pm$ 0.8 & 93.5 $\pm$ 4& 149.4 $\pm$ 4.5 & 442.5 $\pm$ 8.8 \\
        LMC-qEHVI & 25.5 $\pm$ 1.5 & 23.9 $\pm$ 0.5 & 24.7 $\pm$ 0.7 & \textbf{116 $\pm$ 4.6} & \textbf{347.2 $\pm$ 7.3} \\
        BOLFI & 1.1 $\pm$ 0& 1 $\pm$ 0 & 1.7 $\pm$ 0.1 & 129.3 $\pm$ 11.7 & 369.5 $\pm$ 16.5 \\
        SNPE & \textbf{0.1 $\pm$ 0} & \textbf{0.1 $\pm$ 0} & \textbf{0.4 $\pm$ 0} & 168.6 $\pm$ 4.4 & 710.9 $\pm$ 415 \\
        SNLE & 119.4 $\pm$ 2.9 & 137.3 $\pm$ 5.8 & 355.3 $\pm$ 23.1 & 582.7 $\pm$ 23.8 & 1003.4 $\pm$ 92.4 \\
        SNRE & 34$\pm$ 1.5 & 35.4 $\pm$ 0.3 & 110.2 $\pm$ 6.1 & 309.6 $\pm$ 8.6 & 446.4 $\pm$ 14 \\
    \end{tabular}
    \label{tab:Table-3}
\end{table*}

\begin{figure*}
    \centering
    \subfloat[UMAP parameterization model]{\label{fig:umapexp}\includegraphics[height=5cm]{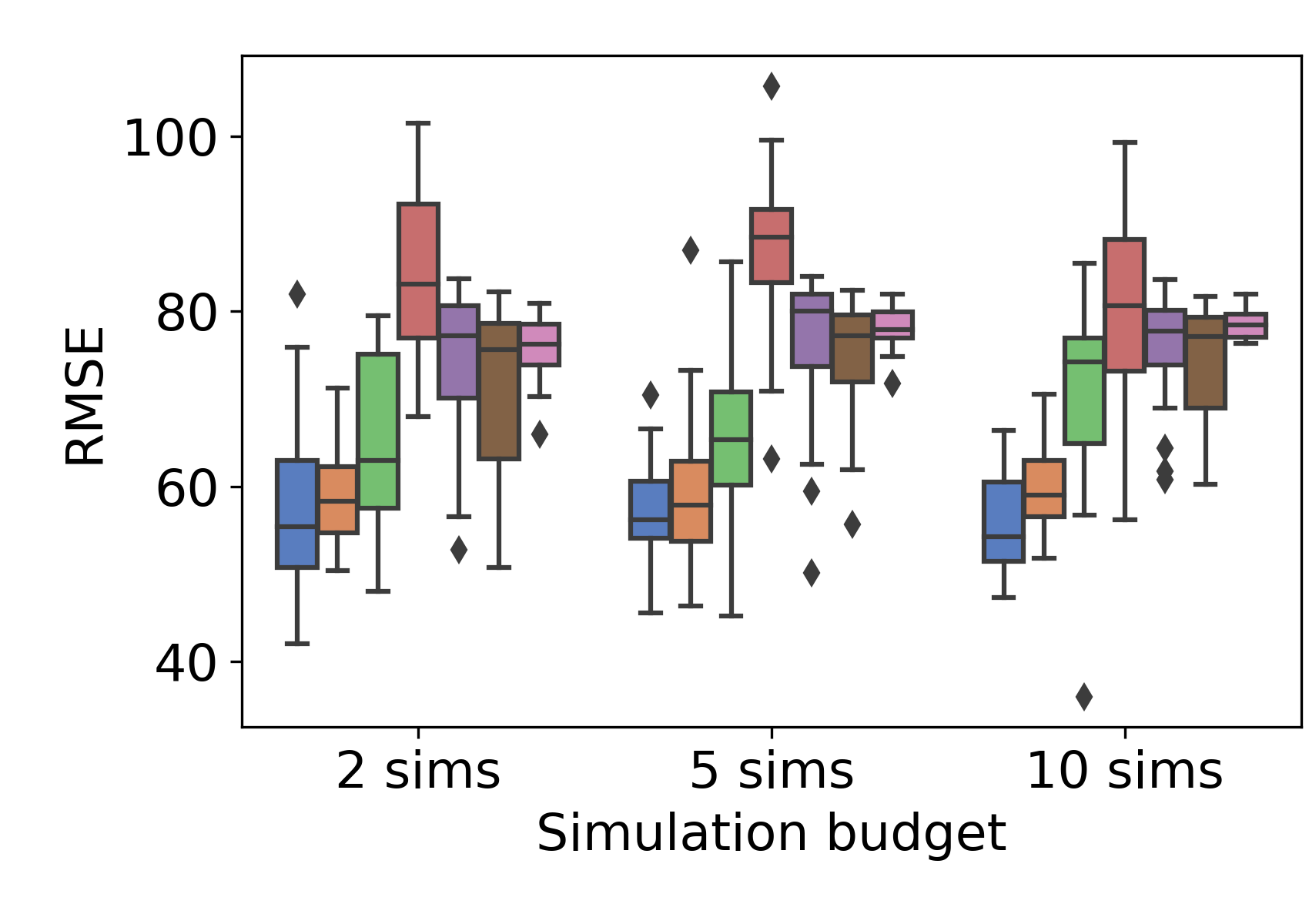}}
    \subfloat[Eye movement control for gaze-based selection]{\label{fig:gazeexp}\includegraphics[height=5cm]{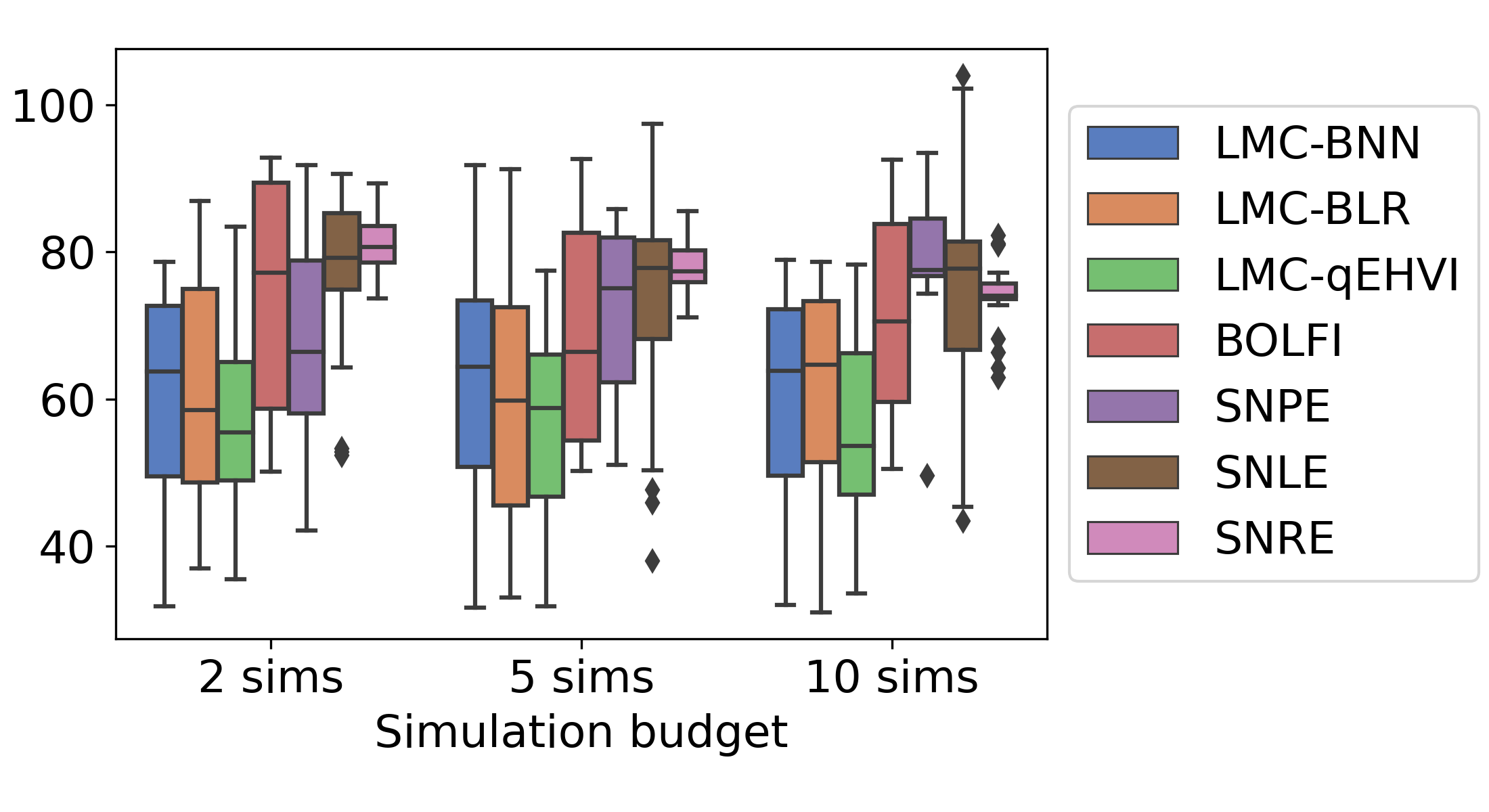}}
    
    \caption{The performance of LFI methods for the state inference tasks with various simulation budgets in two non-stationary user modelling experiments. The box plots were computed from 30 repetitions with different random seeds. The horizontal line on box plots shows the median, the bar shows upper and lower quartiles, and the whiskers indicate the rest of the quartiles. The diamond points indicate outliers.}
    
    \label{fig:expres}
\end{figure*}

The results for the inference and prediction tasks are presented in Tables \ref{tab:Table-1} and \ref{tab:Table-2}, respectively. The lower the RMSE, the better is the quality of estimation. In the inference task, the proposed LMC-based methods clearly outperformed BOLFI and SNE approaches. This indicates that considering multiple objectives at the same time was beneficial for the state inference, and that the model actually leverages information from consecutive states without hindering the performance. Additionally, it can be seen that all LMC-based variants performed differently, which can only be attributed to how the next simulations were chosen, since the surrogate was exactly the same in all three methods. As the results show, having BNN as a model for state transition was beneficial for experiments with non-stationary user models, while having BLR was more preferable for the simpler models. This suggests that BLR is expressive enough to replicate simple transitions, but struggles with more complex ones, for which BNN was more suitable.

The comparisons with GP-SSMs and PR-SSMs for learning transition dynamics show that our method learns accurate dynamics or, at least, relative to the SSM method baselines. The SSMs methods showed worse results than BLR and BNN approaches. This can be explained by the lack of observations for learning state transitions by the SSM methods, which also explains the high variance in the sampled trajectories from these methods. As for comparisons between BLR and BNN, BLR performs better only in LG and SV models, while BNN performs better in more complex case studies. Moreover, it should be noted that trajectory sampling from BLR is possible only by retaining all local linearizations of the dynamics, which is a far more limiting approach than having one single model. Therefore, BNN is a more preferable transition dynamics model. 

The empirical time costs for running the LFI methods are shown in Table \ref{tab:Table-3}. It can be seen that the SNPE method was the fastest for the computationally cheap simulators (of SSMs with tractable likelihoods), while the LMC-qEHVI required the least amount of time for the non-stationary user models. This is expected, since the SNEs learn the model only once, and then simply use it for all observations, which is suitable for the computationally cheap simulations with simple LFI solutions. However, for the non-stationary user models, where there are no closed-form likelihoods available, learning a single model actually requires much more time. To summarize, the LMC variants are clearly preferable for the computationally heavy simulators, which dominate the cost of training a transition dynamics model and a multi-objective surrogate.

Finally, Figure \ref{fig:expres} shows how the performance of the LFI models changes with different simulation budgets: 2, 5, and 10 simulations per each time-step. As expected, in general, all methods improved their performance with increased budgets. However, there is little difference in how these methods compare with respect to each other. This indicates that the results are not sensitive to the precise simulation budget. 

In all experiments, we attribute the success of the proposed LMC-BNN method to a more flexible multi-output surrogate, and a more efficient way of choosing simulation candidates. The LMC allows multi-fidelity modelling (e.g. decomposing a stochastic process into processes with different length-scales), which allows leveraging information from multiple consecutive time-steps, unlike standard GPs. At the same time, samples from the transition model provide better candidates for simulations than the alternatives. The flexible surrogate along with adaptive acquisition make our method particularly suitable for online settings, where only a handful of samples are possible per time-step.

\section{Discussion}
\label{sc:discussion}

We proposed an approach for state inference and prediction in the challenging SSM setting, where the transition dynamics are unknown and observations can only be simulated. Importantly, our model of transition dynamics was obtained with few simulations, making it suitable for cases with computationally expensive simulators. This is important because typically sample-efficient LFI approaches discard any temporal information from observed time-series, and cannot do state prediction, which is necessary for choosing the next simulations when simulation budget is limited. We proposed a solution for both these challenges: we use a multi-objective surrogate model for the discrepancy measure between observed and synthetic data, which connects the consecutive states through shared parameters, and we train an additional surrogate for state transitions with samples from LFI state posteriors. Additionally, our method does not restrict the family of admissible solutions for the state transitions to be linear or Gaussian, unlike existing LFI methods for SSMs \citep{jasra2012filtering, martin2014approximate}, making it more widely applicable. 

Although our method uses a more flexible surrogate for LFI of states, we demonstrated that it requires neither additional data nor significantly more training time than traditionally used GP surrogates. We reached the sample-efficiency goal by sharing synthetic observations across all discrepancy objectives, allowing the method to use the same simulations an indefinite amount of times. As for the decreased training time, we proposed a moving window approach that allowed the surrogate to focus only on a few recent SSM time-steps at a time. In conclusion, having a more flexible surrogate improved state inference and provided better samples from state posteriors for learning the unknown dynamics.

The main limitation of our approach is that the proposed transition dynamics model does not account for long-term state dependencies. Our state transition surrogate considers only the most recent state as an input, assuming the Markov property, and therefore cannot forecast far into the future. The resulting predictions have very low variance and have a tendency to converge to similar values, which can be explained by training only on a single trajectory. This limits our method to cases where observations are highly informative.


\section*{Acknowledgments}
This work was supported by the Academy of Finland (Flagship programme: Finnish Center for Artificial Intelligence FCAI; grants 328400, 319264, 292334) and UKRI Turing AI World-Leading Researcher Fellowship EP/W002973/1. HP was also supported by European Research Council grant 742158 (SCARABEE, Scalable inference algorithms for Bayesian evolutionary epidemiology). Computational resources were provided by the Aalto Science-IT Project.

\subsubsection*{Code Availability}

All code is available through the link: https://github.com/AaltoPML/LFI-in-SSMs-with-UD

\bibliographystyle{unsrtnat}
\bibliography{main}

\appendix
\newpage

\section{Moving window experiments}
\label{sec:app-mov}
We summarize the performance of our multi-objective LFI approach with various moving window lengths in Tables \ref{tab:SM-Table-1}, \ref{tab:SM-Table-2} and \ref{tab:SM-Table-3}, where the moving window size is shown in parentheses after the method's name. The main finding of these experiments is that including only two objectives inside the moving window is sufficient to get the most performance benefits, while having more objectives leads to increased computation time and inconsistent performance results. The increase in computation (Table \ref{tab:SM-Table-3}) is evident from the complexity analysis in Section \ref{sec:complexity}, as it makes the multi-objective surrogate training time to grow linearly. As for the inconsistent performance, the results become worse when transition dynamics have a rapid change rate, which makes the consequent objectives dissimilar and, hence, multi-objective modelling more difficult. For example, LMC-BNN and LMC-qEHVI improved for the LG and NN cases (see Tables \ref{tab:SM-Table-1}, \ref{tab:SM-Table-2}), which had mostly smooth trajectories, but struggled with the SV case, which had occasional erratic transitions, as shown in Figure \ref{fig:svres}. These results indicate that one should choose the moving window size according to the transition dynamics' change rate, which is difficult to determine when dynamics are unknown. Therefore, because any supposed improvements of having bigger moving windows are small and computationally more costly, when dealing with unknown dynamics we recommend setting moving window size to two objectives.

It is also noteworthy to mention how different acquisition methods are impacted by having additional objectives. Starting with LMC-BNNs, with more objectives, the transition dynamics model gets samples from more frequently updated state posterior approximations, since each objective stays longer inside a moving window and their corresponding posteriors are updated every time the window moves. This is beneficial for simple transition dynamics where the multi-objective surrogate maintains a good fit with high number of objectives (e.g. LG, NN), but not for the SV case, where the surrogate struggles with modelling a lot of objectives. At the same time, these performance losses and gains remain negligible. The next acquisition method, LMC-qEHVI, is not impacted by the quality of the extracted posteriors. Instead, it prioritizes optimization of objectives inside the moving window without taking into the account future states. Specifically, it works only when future objectives are very similar to the current ones, and completely fails when transition dynamics have erratic behaviour (as in SV, UMAP and Gaze). Lastly, the BLR acquisition linearizes the transition dynamics locally and when this locality increases by including more objectives, the considered locality becomes less and less linear. As a result, increasing the moving window size hinders the performance of the LMC-BLR in all cases. In conclusion, small window size for all three methods is preferable.

\begin{table*}
    \centering
    \caption{Comparison of LFI methods (rows) in different SSMs (columns) for the state inference task with different moving window lengths (a number in parentheses). The performance was measured with $95\%$ confidence interval (CI) of the RMSE between estimated vs ground truth state points for 50 time-steps. The best results in each column are highlighted in bold.}
    \tiny
    \begin{tabular}{cccccc}
        \\
        \hline
        \textbf{Methods} & \textbf{LG} & \textbf{NN} & \textbf{SV} & \textbf{UMAP} & \textbf{Gaze} \\  \hline
    
        LMC-BNN (2) & 1.77 $\pm$ 0.12 & 6.92 $\pm$ 0.51 & 16.14 $\pm$ 3.27 & 58.24 $\pm$ 3.62 & 58.7 $\pm$ 5.4  \\
        LMC-BNN (3) & 1.78 $\pm$ 0.12 & 6.95 $\pm$ 0.53 & 15.86 $\pm$ 2.6 & \textbf{56.99 $\pm$ 2.68} & 59.45 $\pm$ 5.02  \\
        LMC-BNN (5) & 1.76 $\pm$ 0.11 & \textbf{6.84 $\pm$ 0.49} & 17.6 $\pm$ 2.69  & 58.89 $\pm$ 2.9  & 60.26 $\pm$ 5.31  \\
        
        LMC-BLR (2) & \textbf{1.3 $\pm$ 0.1} & 6.86 $\pm$ 0.54 & 13.15 $\pm$ 3.25 & 59.19 $\pm$ 3.31 & 60.6 $\pm$ 5.8 \\
        LMC-BLR (3) & 1.76 $\pm$ 0.13 & 6.84 $\pm$ 0.65 & \textbf{12.14 $\pm$ 2.46} & 61.68 $\pm$ 4.22 & 60.49 $\pm$ 5.01  \\
        LMC-BLR (5) & 1.8 $\pm$ 0.16 & 7.12 $\pm$ 0.54 & 12.83 $\pm$ 2.78 & 61.8 $\pm$ 5 &  60.35 $\pm$ 5.29 \\
        
        LMC-qEHVI (2) & 1.5 $\pm$ 0.1 & 7.03 $\pm$ 0.55 & 24.4 $\pm$ 2.5 & 64.96 $\pm$ 3.72 & \textbf{ 56.9 $\pm$ 4.5 } \\
        LMC-qEHVI (3) & 1.47 $\pm$ 0.1 & 7.37 $\pm$ 0.53 & 26.02 $\pm$ 2.5 & 68.75 $\pm$ 4.13 & 61.68 $\pm$ 4.85  \\
        LMC-qEHVI (5) & 1.41 $\pm$ 0.07 & 6.92 $\pm$ 0.64 & 26.64 $\pm$ 2.7 & 67.42 $\pm$ 2.92 & 60.76 $\pm$ 4.74  \\
        
    \end{tabular}
    
    \label{tab:SM-Table-1}
\end{table*}

\begin{table*}
    \centering
    \caption{Comparison of transition dynamics models (rows) in different SSMs (columns) with different moving window lengths (a number in parentheses). The performance was measured with $95\%$ CI of the RMSE between sampled vs ground truth trajectories of length 50. The best results in each column are highlighted in bold.}
    \begin{tabular}{cccccc}
    \\
        \hline
        \textbf{Methods} & \textbf{LG} & \textbf{NN} & \textbf{SV} & \textbf{UMAP} & \textbf{Gaze} \\ \hline
        LMC-BNN (2) & 210 $\pm$ 4 & 148 $\pm$ 2 & 117 $\pm$ 21 & 1394 $\pm$ 27 & \textbf{1365 $\pm$ 3} \\ 
        LMC-BNN (3) & 209 $\pm$ 3 & 147 $\pm$ 2 & 123 $\pm$ 25 & 1408 $\pm$ 28 & 1367 $\pm$ 2  \\
        LMC-BNN (5) & 208 $\pm$ 3 & \textbf{147 $\pm$ 1} & 134 $\pm$ 27 & 1416 $\pm$ 36 & 1366 $\pm$ 2  \\
        
        LMC-BLR (2) & \textbf{ 64 $\pm$ 7} & 154 $\pm$ 4 & 100 $\pm$ 37 & 1409 $\pm$ 49 & 1372 $\pm$ 3 \\
        LMC-BLR (3) & 197 $\pm$ 6 & 153 $\pm$ 4 & \textbf{79 $\pm$ 22} & \textbf{1387 $\pm$ 56} & 1367.2  \\
        LMC-BLR (5) & 197 $\pm$ 5 & 150 $\pm$ 2 & 84 $\pm$ 22 & 1404 $\pm$ 67 & 1370 $\pm$ 5 \\

    \end{tabular}
    
    \label{tab:SM-Table-2}
\end{table*}

\begin{table*}
    \centering
    \caption{Time comparison of LFI methods (rows) in different SSMs (columns) with different moving window lengths (a number in parentheses) for training 50 time-steps. The running time is shown in minutes along with $95\%$ CI. The best results in each column are highlighted in bold.}
    \scriptsize
    \begin{tabular}{cccccc}
        \\
        \hline
        \textbf{Methods} & \textbf{LG} & \textbf{NN} & \textbf{SV} & \textbf{UMAP} & \textbf{Gaze} \\  \hline
        
        LMC-BNN (2) & 87.6 $\pm$ 2.6 & 79.2 $\pm$ 1 & 81 $\pm$ 2.9 & 171.2 $\pm$ 5 & 408 $\pm$ 8 \\
        LMC-BNN (3) & 105.8 $\pm$ 5.8 & 96.8 $\pm$ 2 & 84.9 $\pm$ 5.7 & 201.4 $\pm$ 6.2 & 457.9 $\pm$ 6.4  \\
        LMC-BNN (5) & 93.8 $\pm$ 2.2 & 97.3 $\pm$ 2.1 & 102.4 $\pm$ 3.6 & 198.2 $\pm$ 6.1 &  434.4 $\pm$ 8.1  \\
        
        LMC-BLR (2) & 82.1 $\pm$ 5.5 & 48.2 $\pm$ 0.8 & 93.5 $\pm$ 4 & 149.4 $\pm$ 4.5 & 442.5 $\pm$ 8.8 \\
        LMC-BLR (3) & 82.9 $\pm$ 1.6 & 58.5 $\pm$ 4.1 & 112.2 $\pm$ 1.4 & 189.5 $\pm$ 7.5 & 496.7 $\pm$ 13.2  \\
        LMC-BLR (5) & 64 $\pm$ 1.8 & 62.8 $\pm$ 2.7 & 111.6 $\pm$ 2.2 & 190.3 $\pm$ 6.4  &  513.5 $\pm$ 11.5 \\
        
        LMC-qEHVI (2) & 25.5 $\pm$ 1.5 & \textbf{23.9 $\pm$ 0.5} & \textbf{24.7 $\pm$ 0.7} & \textbf{116 $\pm$ 4.6} & \textbf{347.2 $\pm$ 7.3} \\
        LMC-qEHVI (3) & \textbf{24.7 $\pm$ 1.5} & 27.7 $\pm$ 0.9 & 25.4 $\pm$ 1 & 159 $\pm$ 6 & 340.2 $\pm$ 7.1 \\
        LMC-qEHVI (5) & 28.6 $\pm$ 0.8 & 41.3 $\pm$ 18.3 & 31.7 $\pm$ 0.4 & 141.6 $\pm$ 7.6  & 394.5 $\pm$ 10.1  \\

    \end{tabular}
    \label{tab:SM-Table-3}
\end{table*}

\section{The state-space models with tractable likelihoods}
\label{sec:app-simdesc}

In this section, we present details on the three SSMs with tractable likelihoods that were used to assess the quality of state transition models in the experiments. For all three SSMs, we assume the ground-truth transition dynamics and define an observation model along with priors for LFI of states. As observations, we used datasets of 10 points, and their mean and standard deviation as summary statistics. The objective in BO was the logarithm of Euclidean distance between the observed and simulated summary statistics.

\paragraph{Linear Gaussian (LG).} In the LG model, state transition dynamics (Figure \ref{fig:lgssmres}) and an observation model are both linear 
\begin{align}
    \theta_t &= 0.95 \times \theta_{t-1} + v_t, \quad x_t = \theta_t + w_t,
\end{align}
with added Gaussian noise $v_t \sim \mathcal{\N}(0, 1^2)$ and $w_t \sim \mathcal{\N}(0, 10^2)$. The initial state point for the transition dynamics is $\theta_0 = 100$. The prior for the states is $\theta \sim \text{Unif}(0, 15) \in \R$.

\paragraph{Non-linear non-Gaussian (NN).} The NN model is a popular non-linear SSM \cite{kitagawa1996monte}, where the transition dynamics (Figure \ref{fig:toyres}) and observation model are
\begin{align}
    \theta_{h,t} &= \frac{\theta_{h,t-1}}{2} + 25\frac{\theta_{h,t-1}}{\theta_{h,t-1}^2 + 1} + 8\cos(1.2t) + v_t, \quad x_t = \frac{\theta_{h,t}^2}{20} + w_t,
\end{align}
with added Gaussian noise $v_t \sim \mathcal{N}(0, 10)$ and $w_t \sim \mathcal{N}(0, 10)$. The initial state point for the noise standard deviation $\theta_{h, 0}$ is $0$ with the prior $\theta_h \sim  \text{Unif}(-30, 30) \in \R$.

\paragraph{Stochastic volatility (SV)} models are widely used for predicting volatility of asset prices \citep{taylor1994modeling,shephard1996statistical}. Here, we use the model by \cite{barndorff2002econometric} that specifies transition dynamics (Figure \ref{fig:svres}) of volatility $\theta_{v}$ as
\begin{align}
    \theta_{v, t+1} &= (z_t - z_{t+1} + \sum_{j=1}^k e_j), \quad
    z_{t+1} = \text{exp}(-\lambda)z_t + \sum_{j=1}^k\text{exp}\{-\lambda(t+1-c_j)\}e_j,
\end{align}
with $c_{1:k} \sim \text{Unif}(t,t+1)$, $e_{1:k} \sim \text{Expon}(0.5/0.25^2)$ and $\lambda=0.01$. The random increases of volatility are regulated by the Poisson distributed variable $k \sim \text{Poisson}(0.5\lambda^2/0.25^2)$. The observation model for the log-return of an asset $x_t \in \R$ follows
\begin{align}
    x_t = \theta_\mu + \theta_\beta \theta_{v, t} + (\theta_{v, t} ^ {0.5} + 10^{-5}) \varepsilon_t,
\end{align}
where $\varepsilon_t \sim \mathcal{N}(0, 1)$; $\theta_\mu=0$ and $\theta_\beta=0$ remain the same, while the volatility $\theta_{v}$ follows the transition dynamics, starting with the initial value for the volatility $\theta_{v,0}$ of $1$. We set the following priors for states: $\theta_\mu \sim \text{Unif}(-2, 2) \in \R, \quad \theta_\beta \sim \text{Unif}(-5, 5) \in \R, \quad \theta_{v} \sim \text{Unif}(0, 3) \in \R.$

\begin{figure}
    \centering
    \subfloat[Linear Gaussian (LG)]{\label{fig:lgssmres} \includegraphics[width=0.5\textwidth]{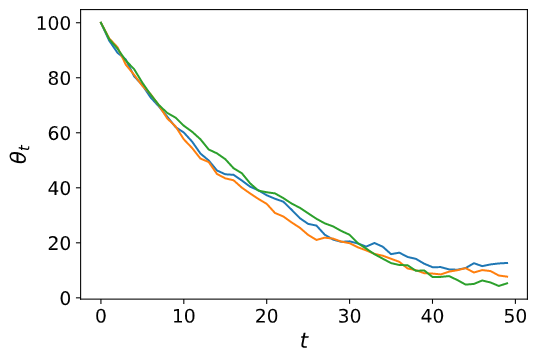}} \\
    \subfloat[Non-Linear Non-Gaussian (NN)]{\label{fig:toyres}\includegraphics[width=0.5\textwidth]{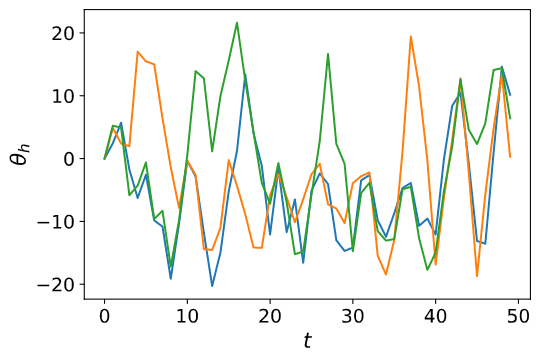}}\\
    \subfloat[Stochastic Volatility (SV)]{\label{fig:svres}\includegraphics[width=0.5\textwidth]{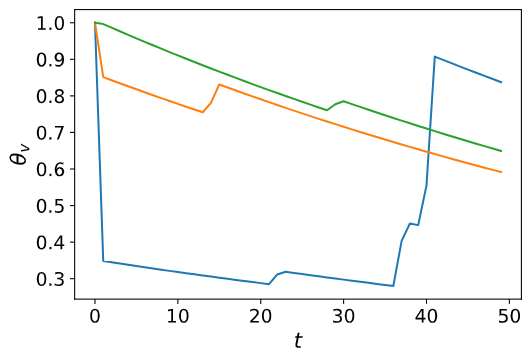}}
    
    \caption{Transition trajectories (different colours) of states (y-axis) sampled from three SSM transition dynamics across 50 time-steps (x-axis) with different random seeds. The complexity of the dynamics gradually increases in SSMs, starting with the smooth LG (a) dynamics, where the difference between consecutive states is very small, followed by NN (b) with more erratic behaviour, and finishing with the SV model (c), whose dynamics has occasional drastic changes in state values.}
    \label{fig:s1}
\end{figure}

\section{Technical details} 
\label{sec:app-tech-details}

\subsection{Eye movement control for gaze-based selection}

An observation model in the gaze selection experiment is a simulated environment, where a human-subject is modelled through a reinforcement learning agent. In this environment, the agent searches for a target on a 2D display, where the target location, actions, observations, and beliefs of the agent are represented by two coordinates $\{c_1, c_2\}$, $c_1, c_2 \in [-1, 1]$ on the display. At each episode of the environment $e$, the agent receives noisy observations of the target $\bo_e = \N(\bq, \theta_{s} \times E^{(e)})$ and moves the gaze to a new location $\ba_e = \N( PPO(\bo_e, \theta_{om} \cdot A^{(e)})$. The beliefs $\bb_e$ about the target location $\bq$ are updated according to
\begin{align}
    \bb_{e+1} = \frac{ \sigma_{(o, e+1)}^2  \bo_{e+1} + \sigma_{(b, e)}^2 \bb_{e}   }{\sigma_{(o, e+1)}^2 + \sigma_{(b, e)}^2}, \quad \sigma_{(b, e+1)} = \sqrt{ \frac{\sigma_{(o, e+1)}^2 \sigma_{(b, e)}^2}{ \sigma_{(o, e+1)}^2 + \sigma_{(b, e)}^2 } }.
\end{align}
where $\sigma_{(o, e)}$ and $\sigma_{(b, e)}$ are observation and belief uncertainties respectively, $A^{(e)}$ is the amplitude, and $E^{(e)}$ is the eccentricity of the saccade at the episode $e$. The user model was trained on 10 000 episodes using the Proximal Policy Optimization (PPO) algorithm \citep{schulman2017proximal}, provided by the Stable Baselines3 library \citep{stable-baselines3}. We used default parameters, a multilayer perceptron policy and a clipping parameter set to 0.15. The environment was implemented by Chen et al. \citep{chen2021adaptive} in Python with the Open AI gym framework \citep{open-ai-gym}.

\subsection{UMAP parameterization}

In the UMAP parameterization model, the ground truth for states is not available, because the human-subject cannot specify the ideal embeddings. Therefore, we approximate the ground truth by using ABC with rejection sampling. Usually, this requires running millions of simulations for each time-step, but we make use of the weighted form of the preference score that allows us to use the same simulations across all time-steps. We simulate $1,500,000$ embeddings with state points sampled from the prior, and then calculate their corresponding $\calUi$ and $\calPi$ from Equation 4.1 of the main text. For each time-step $t$, we calculate the preference scores $\bd$ and retain only $0.06\%$ of those states that showed the lowest $\bd$ values. Then, we use a Gaussian kernel density estimator (KDE) on the retained state points, and maximize the corresponding PDFs to find the estimations of the ground truth. The bandwidth $b$ of the kernel was calculated according to a Scott's rule \citep{scott2015multivariate} $b = n^{-\frac{1}{d+4}}$, where $n$ is the number of data points and $d$ is the number of dimensions.

The preference score was computed as a weighted sum between the relative validity $\calUi$ and the classification accuracy $\calPi$ on the validation set. The relative validity $\calUi$ is an approximation of the Density Based Cluster Validity (DBCV) score \citep{moulavi2014density}, which is often used as a quality measure of clustering. Intuitively, it shows how separable all the clusters are. We use the HDBSCAN* package \citep{McInnes2017} and the HDBSCAN Boruvka KDTree \citep{campello2013density} algorithm to cluster the points of the embeddings. We set the smallest size grouping to 60, the density parameter of clusters to 10 and leave the rest parameters to their default values. The classification accuracy $\calPi$ was calculated for the C-Support Vector Classifier (SVC) \citep{boser1992training, cortes1995support} with the Scikit-learn package \citep{scikit-learn} and default parameters.

The embeddings for the UMAP parameterization model were computed for the handwritten digit dataset \citep{alpaydin1998cascading}. The dataset was randomly split on the training and validation sets, resulting in $1198$ and $599$ 8x8 pixel images of digits in 10 digit classes. The UMAP algorithm was implemented in the UMAP-learn package \citep{mcinnes2018umap-software}.

\subsection{Implementation details of methods}

All methods in the following subsections were  integrated in the Engine for Likelihood-Free Inference (ELFI) \citep{JMLR:v19:17-374} with the code available for application and further development through the link: https://github.com/AaltoPML/LFI-in-SSMs-with-UD.

\subsubsection{BOLFI} 
\label{sec:sm-bolfi}

For BOLFI, we used GP surrogate in BO with the LCBSC acquisition function. The posterior samples were extracted with importance weighted resampling. Below are the technical details for each of these BOLFI components:

\paragraph{GP surrogate.} The implementation for the GP surrogate was provided by the GPy package \citep{gpy2014}. It was initiated with zero mean function, and with the following priors for the RBF kernel lengthscale $\boldsymbol{l}$, its variance $\sigma_k^2$, and the variance of the bias term $\sigma_b^2$:
\begin{align}
    l &\sim \text{Gamma}\left(\frac{\bt_\text{max} - \bt_\text{min}}{3}, \boldsymbol{1}\right), \quad \sigma_k^2 \sim \text{Gamma}\left(\frac{\text{max}(\tdi)^2}{9}, 1\right), \nonumber \\
    \sigma_b^2 &\sim \text{Gamma}\left(\frac{\text{max}(\tdi)^2}{36}, 1\right),
\end{align}
where $\bt_{min}, \bt_{max}$ are the lower and upper bounds for $\bti$, ($\bti, \tdi$) are initial training points, and $\boldsymbol{1}$ is an all-ones vector. The GP was minimized by using the SCG optimizer \citep{andrei2007scaled} on the GP negative log-likelihood with a maximum number of iterations of 50. All inputs $\bti$ for the GP were centred and normalized.

\paragraph{LCBSC acquisition.} The BOLFI implementation uses LCBSC \citep{srinivas2009gaussian, brochu2010tutorial} as an acquisition function, which chooses points that minimize the lower confidence bound (LCB)
\begin{align}
    \text{LCB}(x) = \mu(x) - \beta_t^{1/2}\sigma(x), \quad \beta_t = 2\text{log}(t^{2d + 2}\pi^2/3\delta),
\end{align}
where $\beta_t$ is the confidence parameter, $\delta = 0.1$ is the inverse exploration rate and $d$ is an input dimension.

\paragraph{Posterior sampling.} The BOLFI posterior was obtained by weighting the prior samples and using corresponding  likelihoods as weights
\begin{align}
    p(\bxobs | \bt) \propto F\left(\frac{\epsilon - \mu(\bt)}{\sqrt{\nu(\bt) + \sigma^2}} \right) \label{eq:unnormll},
\end{align}
with $F(\cdot)$ being a Gaussian CDF with the mean 0 and the variance 1, where $\epsilon$ is the minimum of the GP surrogate mean function $\mu(\cdot)$ obtained by using the L-BFGS-B optimizer \citep{byrd1995limited} with a maximum number of iterations of 1000.

\subsubsection{Muti-objective LFI with transition model} 
\label{sec:sm-molfi}

For our method, we used an LMC surrogate with the proposals for simulations coming from a Bayesian Neural Network (BNN) transition model. In the experiments, we also used the qEHVI multi-objective acquisition function and the Bayesian Linear Regression (BLR) alternatives for the transition model component. The technical details for the main components as well as for their alternatives are described below along with a recommended default set of parameters, which performed consistently well in all experiments. Algorithm \ref{alg:lmcbnn-sm} follows the algorithm from the main section, but provides a lower-level description of the method.

\begin{algorithm}
   \caption{Multi-Objective LFI with Transition Model}
   \label{alg:lmcbnn-sm}
\begin{algorithmic}
    \STATE {\bfseries Input:} observations $\{\bx_1, ..., \bx_T \}$, observation simulator $g_{\bt}$, moving window size $L$, number of initial simulations $B_0$, simulation budget $B_\text{sim}$, number of posterior samples $I$ for states and $K$ for states predictions, prior over states $p(\bt_*)$;
    \STATE {\bfseries Output:} state posteriors $\mathcal{P}$, transition model $\widetilde{h}_{\bt_t}$;
    \STATE
    \STATE Initialize an empty set for state posteriors $\mathcal{P} := \emptyset$; 
    \STATE Initialize the transition model $\widetilde{h}_{\bt_t}$ (see Section \ref{sec:learningstate});
    \STATE Simulate $B_0$ observations in $\mathcal{S}$ with parameters sampled from the prior $p(\bt_*)$: 
    \STATE\hspace{.25cm} $\mathcal{S} := \{(\bt_*, \bx_{\bt}): \bx_{\bt} = g_{\bt}(\bt_*), \bt_* \sim p(\bt_*) \}^{B_0}_{b=1}$;
    \STATE Initialize the temporary variable for posterior samples $\phi_\text{old} := \emptyset$;
    \STATE Initialize the start and end indexes for the moving window: $t_0 := 0$, $t := L$;
    \STATE
    \WHILE{$t < T$}
        \STATE Calculate the discrepancies for parameters in $\mathcal{S}$ and form the training set $\mathcal{T}$:
        \STATE\hspace{.25cm} $ \mathcal{T} := \{ \{ (\bt_*, \delta_w(\bt_*)) \}_{w=t_0}^t : \forall (\bt_*, \bx_{\bt}) \in \mathcal{S} \}^{|\mathcal{S}|}_{s=1} $;
        \STATE Train the multi-objective model $\widetilde{\delta}_{\bt}$ from Section \ref{sec:mostateinference} with $\mathcal{T}$;
   
        \STATE Extract $I$ posterior samples over states from $\widetilde{\delta}_{\bt}$ with \eqref{eq:post}:
        \STATE\hspace{.25cm} $\phi^{t_0:t} := \{ \{ \bt_w : \bt_w \sim p(\bt_w \mid \bx_w) \}_{w=t_0}^t \}^{I}_{i=1}$; 
        
        \STATE Use extracted samples to update the set of posteriors by replacing $\phi_\text{old}$ in $\mathcal{P}$: 
        \STATE\hspace{.25cm} $\mathcal{P} := (\mathcal{P} \setminus \phi_\text{old}) \cup \phi^{t_0:t}$, set $\phi_\text{old} := \phi^{t_0+1:t}$;

        \STATE Form the training set $\mathcal{D}$ by pairing $K$ consecutive state samples from $\mathcal{P}$: 
        \STATE\hspace{.25cm} $\mathcal{D} := \{\{ (\bt_{j}, \bt_{j+1}): \bt_{j} \in_R \phi^{j}, \phi^{j} \in \mathcal{P}\}_{j=1}^t \}^K_{k=1}$;
    
        \STATE Update parameters of $\widetilde{h}_{\bt_t}$ with data from $\mathcal{D}$ through training;
    
        \STATE Simulate $B_\text{sim}$ new observations from parameters proposed by $\widetilde{h}_{\bt_t}$ and update $\mathcal{S}$: 
        \STATE \hspace{.25cm}$\mathcal{S} := \mathcal{S} \cup \{(\bt_*, \bx_{\bt}) : \bx_{\bt}=g_{\bt}(\bt_*), \bt_* \sim \widetilde{h}_{\bt_t}(\cdot \mid \bt_{t}) \}^{B_\text{sim}}_{b=1}$;
        
        \STATE Move the moving window by adjusting its indexes: $t_0 := t_0 + 1$, $t := t + 1$;
    \ENDWHILE
\end{algorithmic}
\end{algorithm}

\paragraph{LMC surrogate.} The LMC model was implemented in BoTorch \citep{balandat2020botorch}. Its latent GPs were initialized with linear mean functions $\mu(x) = k x + b$, and RBF kernels. The lengthscales of the kernels were parameterized in log scale, and initialized with 0 along with the constant $b$ of the mean function. For the RBF kernel, ARD was also enabled to assign separate lengthscales for each input dimension. Furthermore, GP latents were initialized with 50 inducing points uniformly sampled inside the input bounds for each latent GP. The LMC training used Adam optimizer from the tensor computation package PyTorch \citep{paszke2019pytorch} with a learning rate of 0.1 to minimize the variational evidence lower bound (ELBO). The optimized parameters included LMC coefficients, inducing points locations, and hyperparameters for the mean functions and kernels. Lastly, we used the default training step size and 1000 epochs for training with all inputs for the LMC centred and normalized.

\paragraph{qEHVI acquisition.} The qEHVI acquisition function used a Quasi-MC sampler \citep{caflisch1998monte} with scrambled Sobol sequences \citep{owen1998scrambling} and a sample size of 128. The reference point that was used for calculating the hypervolume for each objective was set to -5. The acquisition points were acquired sequentially using successive conditioning \citep{wilson2018maximizing} with a maximum number of restarts of 20, a batch size limit of 10, and a maximum number of iterations for the optimizer of 200. It was also implemented in BoTorch.

\paragraph{Bayesian neural network (BNN)} \citep{blundell2015weight} was built from stacked Bayesian layers implemented in torch zoo (BLiTZ) \citep{esposito2020blitzbdl}. We used an architecture with 2 hidden layers, where each layer had 256 nodes. During the training process, stochastic gradient descent \citep{sutskever2013importance} was used with a learning rate of 0.001 for minimizing the ELBO loss with squared L2 norm. The loss was calculated based on 10 samples from the model, for each of 100 batches of training data in a single epoch. The training data was randomly selected from previously stored approximated posterior samples from all states (1000 samples per each state) with replacement, resulting in a total of 1 000 000 points, where one point was a pair of consecutive state points. The training data was updated each time the moving window moved.

\paragraph{Bayesian linear regression (BLR)} is defined as $\bt_{t+1} = \bt_{t}^T \bbeta + \bepsilon$, where $\bepsilon \sim N(0, \sigma^2 I)$. The hyperparameters $\sigma \in \R$ and $\bbeta \in \R^{m \times m}$ were inferred with maximum likelihood estimation (MLE) \citep{myung2003tutorial} of
\begin{align}
    \label{eq:mle}
    p(\bt_{t+1} | \bt_{t}) \propto \frac{1}{\sigma} \exp{- \frac{(\bt_{t+1} -  \bt_{t} \bbeta)^T(\bt_{t+1} - \bt_{t} \bbeta)}{2 \sigma_{t}^2 }}.
\end{align}
The BLR model was implemented with the probabilistic programming package, PyMC3 \citep{salvatier2016probabilistic} that used 100 random samples from three latest inference steps. The ($\theta_{t-1}, \theta_{t}$) and ($\theta_{t-2}, \theta_{t-1}$) pairs were used as inputs for training, and  the normal distribution was chosen as a prior family of the BLR parameters. The model was fitted by using the No-U-Turn Sampler (NUTS) \citep{hoffman2014no} with two chains of 2000 samples, 2000 tuning iterations, and a target acceptance rate of 0.85. 

\paragraph{Posterior sampling.} Similar to BOLFI, the posterior was sampled by using an importance-weighted resampling. Because the main model was implemented in PyTorch, we used the Adam optimizer to learn the threshold $\epsilon$ with a learning rate of $10^{-4}$ and 100 optimizing iterations. The optimization started at the parameter point that produced a synthetic observation with the smallest discrepancy.

\subsubsection{Sequential neural estimators}
\label{sec:sm-snes}

All three SNEs (SNPE, SNLE and SNRE) and their corresponding surrogate models were implemented in the simulation-based inference \citep{tejero-cantero2020sbi} and PyTorch \citep{paszke2019pytorch} packages with default parameters. In all three methods, the Adam optimizer with the learning rate of $5\times10^{-4}$ and the training batch size of $50$ in $20$ epochs were used for training the surrogate. The total gradient norm was clipped in order to prevent exploding gradients at a value of 5.0, and z-score passing was used for surrogate model inputs and outputs.

For \textbf{SNPE} \citep{greenberg2019automatic} and \textbf{SNLE} \citep{papamakarios2019sequential}, the masked autoregressive flow (MAF) surrogate was used for approximating the posterior $p(\bt|\bx)$ and likelihood, $p(\bx|\bt)$ respectively. The MAF consisted of $5$ transformations with $50$ hidden features in each of $2$ blocks. Each autoregressive transformation had tanh activation along with batch normalization. In contrast, \textbf{SNRE} \citep{durkan2020contrastive} approximated the ratio $\frac{p(\bt,\bx)}{p(\bt)p(\bx)}$, where a residual network \citep{he2016deep} was used as a classifier trained to approximate likelihood ratios. The goal of the classifier was to predict which of the $(\bt,\bx)$ pairs was sampled from $p(\bt,\bx)$ and which from $p(\bt)p(\bx)$. The residual network had $50$ hidden features in $2$ residual blocks with $10$ $(\bt,\bx)$ pairs to use for classification.

\subsubsection{Recurrent state-space models with GP transitions}

In the experiments, we used two variants of recurrent state-space models with GP transitions: \textbf{GP-SSM} with a variationally coupled dynamics and trajectories, in which the variational inference scheme for GP transition dynamics is used \citep{ialongo2019overcoming}, and probabilistic recurrent state-space model \textbf{PR-SSM} \citep{doerr2018probabilistic}, which uses doubly stochastic variational inference for efficient incorporation of latent state temporal correlations. Both methods were implemented in the GPt package \citep{ialongo2019overcoming} with default parameters. The GPs were using Matern32 kernels with linear mean functions, along with 50 randomly sampled inducing points. The number of latent dimensions was set equal to the number of simulator parameters of the observation model. The ELBO loss was calculated from 10 posterior samples and optimized with Adam using a learning rate of $0.001$ in $3000$ iterations.

\end{document}